\documentclass[11pt]{article}

\usepackage[margin=1in]{geometry}
\usepackage{amsmath,amssymb,amsthm}
\usepackage{booktabs}
\usepackage{array}
\usepackage{xcolor}
\usepackage{graphicx}
\usepackage{enumitem}
\usepackage{titlesec}
\usepackage{fancyhdr}
\usepackage{microtype}
\usepackage{float}

\definecolor{deciblue}{HTML}{1F4E79}
\definecolor{decired}{HTML}{B91C1C}
\definecolor{darkgray}{HTML}{374151}

\titleformat{\section}{\Large\bfseries\color{deciblue}}{\thesection}{0.7em}{}
\titleformat{\subsection}{\large\bfseries\color{darkgray}}{\thesubsection}{0.7em}{}
\titleformat{\subsubsection}{\normalsize\bfseries\color{darkgray}}{\thesubsubsection}{0.7em}{}

\newtheorem{definition}{Definition}
\newtheorem{theorem}{Theorem}

\newtheorem{lemma}{Lemma}
\newtheorem{corollary}{Corollary}

\newcommand{\ALLOW}{\mathrm{ALLOW}}
\newcommand{\NONACTION}{\mathrm{NON\mbox{-}ACTION}}
\newcommand{\ESCALATE}{\mathrm{ESCALATE}}
\newcommand{\DEFER}{\mathrm{DEFER}}
\newcommand{\INFO}{\mathrm{REQUEST\mbox{-}INFO}}

\begin{document}

\begin{titlepage}
\centering
\vspace*{1.2cm}
{\Huge\bfseries The Pre-Action Legitimacy Gap in AI Systems\par}
\vspace{0.4cm}
{\Large A Deterministic, Non-Compensatory Decision Boundary for Execution\par}
\vspace{1.2cm}
{\large Gadi Lavi\par}
\vspace{0.15cm}
{\normalsize Independent Researcher, DECIMAG\par}
\vspace{0.15cm}
{\normalsize https://decimag-protocol.netlify.app/\par}
\vspace{0.6cm}
{\normalsize April 2026\par}
\vfill
\begin{minipage}{0.86\textwidth}
\textbf{Abstract.}
AI systems are increasingly capable of initiating real-world actions, yet current architectures lack a formal mechanism to determine whether a decision should be allowed to execute in the first place. Existing paradigms such as authorization, safety alignment, policy enforcement, runtime governance, and statistical certification evaluate identity, compliance, risk, or observed behavior. They often assume that once a decision is permitted, safe enough, or policy-compliant, it may proceed to execution.

This paper defines the \emph{pre-action legitimacy problem}: the missing computational step that determines whether an AI-generated decision has the right to exist as an executable event prior to execution. We introduce a minimal formalization of a \emph{Right-to-Act protocol}: a deterministic, non-compensatory boundary that treats legitimacy as a feasibility condition rather than a score. A decision is allowed to execute only if all required structural constraints are satisfied; no positive signal may compensate for a failed required condition.

We prove a non-equivalence result showing that compensatory scoring systems cannot guarantee pre-action legitimacy. We then position this boundary relative to authorization, safety, runtime governance, and statistical certification, and present a case study in AI-driven account suspension. The contribution is intentionally architectural and formal: it defines a missing decision primitive without prescribing or disclosing any proprietary implementation.
\end{minipage}
\vfill
{\small Preprint. This paper introduces a formal framing for pre-action legitimacy and non-compensatory execution boundaries in AI systems.}
\end{titlepage}

\section{Introduction}

AI systems are moving from advisory outputs toward executable decisions. In agentic and tool-using systems, model outputs may trigger fund transfers, account restrictions, database operations, infrastructure changes, procurement workflows, or other actions with legal, financial, operational, or reputational consequences.

A common architecture can be summarized as follows:

\begin{center}
\textit{AI decision} $\rightarrow$ \textit{validation} $\rightarrow$ \textit{execution}.
\end{center}

The validation step may include access control, safety filters, policy checks, risk scoring, compliance rules, or post-hoc audit. These mechanisms are necessary. However, they do not fully answer a different question:

\begin{quote}
\centering
\textbf{Should this decision be allowed to exist as an executable event at all?}
\end{quote}

This paper argues that this question defines a distinct architectural gap. Existing paradigms evaluate whether a user is allowed, whether an output is harmful, whether a policy is violated, or whether observed system behavior falls within certified risk bounds. A pre-action legitimacy boundary instead evaluates whether the AI-generated decision has earned the right to proceed toward execution before ordinary validation or certification is treated as sufficient.

The goal is not to replace authorization, safety, runtime governance, or certification. The goal is to formalize a missing condition before execution: \emph{the right to act}. We call this gap the \emph{Pre-Action Legitimacy Gap}.

\subsection{Contributions}

This paper makes four contributions:

\begin{enumerate}[leftmargin=1.2cm]
  \item It defines the pre-action legitimacy problem as a distinct computational and architectural problem for AI systems that act in the world.
  \item It formalizes a non-compensatory Right-to-Act decision boundary using minimal constraint notation that avoids implementation-specific details.
  \item It proves a non-equivalence theorem: compensatory scoring systems cannot guarantee pre-action legitimacy when required constraints must hold individually.
  \item It provides a case study showing how the proposed boundary changes the outcome of a high-confidence, policy-compliant account suspension decision.
\end{enumerate}

\section{Background and Related Work}

\subsection{Authorization and access control}

Authorization systems determine whether an identity, role, process, or agent is permitted to access a resource or perform a class of operations. These systems ask: \emph{who is allowed to act?} They do not necessarily determine whether a specific AI-generated decision is structurally legitimate in its context. An authorized actor may still produce an action that should not proceed.

\subsection{AI safety and alignment}

AI safety and alignment research focuses on reducing harmful, unintended, or misaligned behavior in AI systems. This includes model training, reward design, refusal behavior, monitoring, and evaluation. These methods are essential but often probabilistic, model-dependent, or focused on generated content and behavior rather than per-decision execution legitimacy.

\subsection{Runtime governance for AI agents}

Recent work on runtime governance treats agent execution paths as the object of governance. Kaptein, Khan, and Podstavnychy formalize policies over paths, arguing that prompts and static access control cannot fully govern path-dependent behavior in agentic systems \cite{kaptein2026runtime}. This line of work is close to the present paper in its focus on runtime boundaries, but it remains oriented around governance and policy over paths. The present paper isolates a narrower primitive: whether a proposed decision has the right to proceed before execution.

\subsection{Pre-action authorization}

Uchibeke characterizes the pre-action authorization problem for autonomous agents, emphasizing deterministic enforcement before tool calls and signed audit records \cite{uchibeke2026before}. This is directly relevant to the present work. The distinction here is conceptual: pre-action authorization asks whether an action is permitted by policy or authority. Pre-action legitimacy asks whether the decision satisfies required structural conditions such that execution is legitimate in the first place.

\subsection{Statistical certification and AI risk regulation}

Levy and Perl propose a statistical certification framework for black-box AI risk regulation, focusing on bounding system behavior under uncertainty \cite{levy2026bounding}. Their framing is valuable for regulation and certification of systems whose internals are opaque. The present paper complements this approach by addressing the per-decision boundary before execution. Certification can state that a system behaves within acceptable bounds over a distribution; a Right-to-Act layer asks whether a particular decision should become executable.

\subsection{Risk management frameworks}

The NIST AI Risk Management Framework provides guidance for managing AI risks across design, deployment, and use \cite{nist2023airmf}. The EU AI Act establishes a risk-based regulatory framework for AI systems \cite{eu2024aiact}. OWASP has also documented risks associated with agentic AI execution layers and autonomous skills \cite{owasp2025agentic}. These frameworks motivate the need for systematic controls, but they do not by themselves define the non-compensatory computational boundary proposed here.

\section{The Pre-Action Legitimacy Gap}

\begin{definition}[Pre-Action Legitimacy Gap]
The Pre-Action Legitimacy Gap is the absence of a formal mechanism that determines, prior to execution, whether an AI-generated decision has the right to exist as an executable event.
\end{definition}

The gap is easiest to see in cases where a decision is:

\begin{itemize}[leftmargin=1.2cm]
  \item authorized by identity or role;
  \item compliant with an explicit policy;
  \item scored as low risk or high confidence;
  \item still inappropriate, premature, unjustified, or structurally invalid in context.
\end{itemize}

Examples include disabling an account without sufficient contextual verification, executing a payment when authority exists but justification is incomplete, or allowing an AI agent to take irreversible infrastructure action under ambiguous state. In such cases, the failure is not merely a safety failure, a policy failure, or an authorization failure. It is a legitimacy failure at the execution boundary.

\section{Formal Model}

This section presents a deliberately minimal model. It is intended to express the boundary mathematically without exposing any implementation-specific logic, calibration method, or internal gate design.

\subsection{Action and decision space}

Let $A$ denote the set of candidate executable decisions produced by, or routed through, an AI system. Each decision $a \in A$ may correspond to a proposed action such as disabling an account, issuing a payment, escalating a transaction, modifying access privileges, or invoking a tool.

We treat each decision abstractly as a structured object:

\begin{equation}
 a = (x, c, t, s, \tau),
\end{equation}

where $x$ denotes the proposed operation, $c$ denotes context, $t$ denotes target, $s$ denotes scope, and $\tau$ denotes timing. This tuple is illustrative only. The framework does not require this exact representation.

\subsection{Legitimacy constraints}

Let

\begin{equation}
 C = \{C_1, C_2, \ldots, C_n\}
\end{equation}

be a finite set of required legitimacy constraints. Each constraint is a deterministic predicate:

\begin{equation}
 C_i : A \rightarrow \{0,1\}.
\end{equation}

$C_i(a)=1$ means that the required condition holds for decision $a$. $C_i(a)=0$ means that the condition fails.

The specific content of the constraints is domain-specific and not defined in this paper. Examples of constraint categories may include sufficient context, reversibility, justification, authority, timing, proportionality, or evidence sufficiency. These are examples only, not an implementation specification.

\subsection{Right-to-Act decision function}

Define the Right-to-Act decision function:

\begin{equation}
D(a) =
\begin{cases}
\ALLOW, & \text{if } \forall i \in \{1,\ldots,n\},\ C_i(a)=1,\\
\NONACTION, & \text{otherwise.}
\end{cases}
\label{eq:rta}
\end{equation}

In operational systems, $\NONACTION$ may be decomposed into outcomes such as $\DEFER$, $\ESCALATE$, or $\INFO$. For the theoretical argument, it is sufficient to group all non-execution outcomes under $\NONACTION$.

\subsection{Legitimacy as feasibility}

Equation \ref{eq:rta} defines legitimacy as a feasibility condition. The question is not whether the decision receives a sufficiently high score, but whether it lies inside a feasible region:

\begin{equation}
F = \{a \in A : \forall i,\ C_i(a)=1\}.
\end{equation}

Then:

\begin{equation}
D(a)=\ALLOW \iff a \in F.
\end{equation}

This is the central distinction. A Right-to-Act protocol does not optimize a score. It enforces a boundary.

\section{Non-Compensatory Decision Boundary}

\begin{definition}[Non-compensatory system]
A decision system is non-compensatory if a failed required condition cannot be offset by any number of satisfied conditions.
\end{definition}

In the present framework:

\begin{equation}
\exists j \text{ such that } C_j(a)=0 \Rightarrow D(a)\neq\ALLOW.
\end{equation}

This property is essential. It prevents a strong confidence signal, a favorable risk score, or a policy match from compensating for a failed structural condition.

\begin{lemma}[Monotonic rejection]
If $a \notin F$, then increasing the number of satisfied constraints without repairing the failed constraint does not change the outcome.
\end{lemma}

\begin{proof}
If $a \notin F$, then by definition there exists $j$ such that $C_j(a)=0$. Equation \ref{eq:rta} requires all constraints to equal 1 for $\ALLOW$. Any additional satisfied constraints $C_k(a)=1$ for $k\neq j$ do not change the fact that $C_j(a)=0$. Therefore $D(a)=\NONACTION$ remains unchanged.
\end{proof}

\section{Non-Equivalence Theorem}

Many deployed AI decision systems use scoring or threshold logic. We now show why this class of systems cannot, in general, guarantee pre-action legitimacy when legitimacy requires non-compensatory constraint satisfaction.

\begin{definition}[Compensatory scoring system]
A compensatory scoring system maps a decision $a$ to a real-valued score:
\begin{equation}
S(a)=\sum_{i=1}^{m} w_i x_i(a),
\end{equation}
where $w_i\geq 0$, $x_i(a)\in[0,1]$, and the system allows execution when:
\begin{equation}
S(a)\geq \theta.
\end{equation}
\end{definition}

\begin{theorem}[Non-equivalence]
A compensatory scoring system is not equivalent to a non-compensatory Right-to-Act system when at least one required condition must hold individually.
\end{theorem}

\begin{proof}
Let $C_k$ be a required condition in a Right-to-Act system. Construct a decision $a$ such that $C_k(a)=0$ while all other favorable signals are high. In a scoring system, choose features $x_i(a)$ for $i\neq k$ sufficiently high such that:
\begin{equation}
\sum_{i\neq k} w_i x_i(a) \geq \theta.
\end{equation}
Then the compensatory system allows execution because $S(a)\geq\theta$.

However, the Right-to-Act system rejects the same decision because $C_k(a)=0$, so $a\notin F$ and $D(a)=\NONACTION$.

Thus, there exists at least one decision for which the scoring system allows and the non-compensatory system rejects. Therefore the systems are not equivalent.
\end{proof}

\begin{corollary}[No weighting repair]
No choice of weights in a compensatory scoring system can guarantee non-compensatory legitimacy unless the scoring system is transformed into an explicit hard-constraint system.
\end{corollary}

\begin{proof}
If a required condition can be violated while other signals contribute positively, compensation is possible. To eliminate compensation, the violated condition must act as a hard blocker. But this changes the system from pure scoring to a constraint system, which is precisely the Right-to-Act structure.
\end{proof}

\section{Architectural Placement}

Figure \ref{fig:architecture} contrasts the standard execution stack with a Right-to-Act boundary. The figure is conceptual and does not disclose internal implementation.

\begin{figure}[H]
\centering
\includegraphics[width=0.98\textwidth]{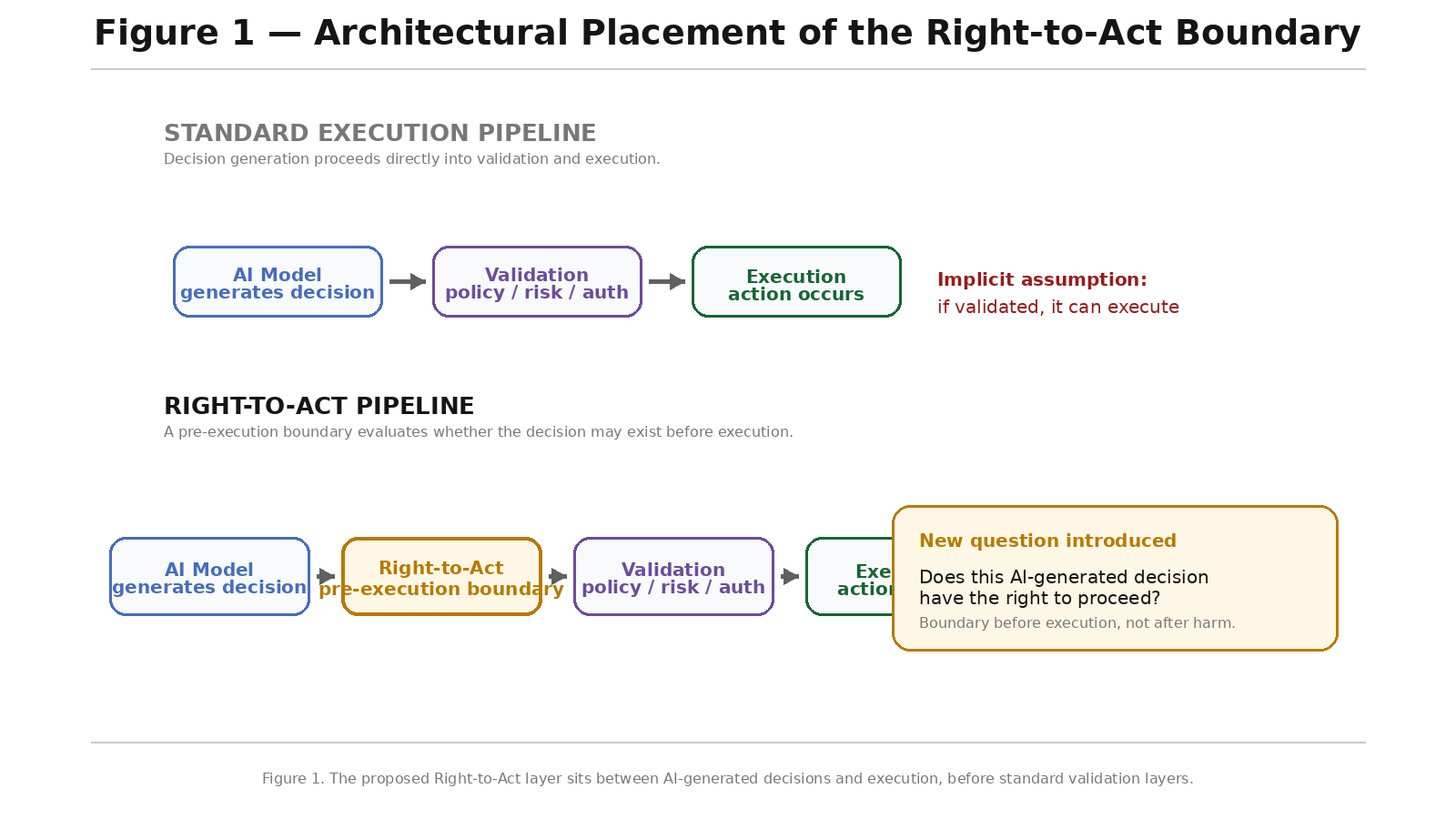}
\caption{The Right-to-Act boundary sits between AI-generated decisions and downstream validation or execution. It asks a different question from authorization, safety, and certification: whether the decision itself has the right to proceed.}
\label{fig:architecture}
\end{figure}

\section{Constraint vs. Score}

Figure \ref{fig:boundary} provides a geometric intuition. Scoring systems may allow actions that exceed a threshold even when a required condition fails. A non-compensatory boundary defines a feasible region, and only decisions inside that region may proceed.

\begin{figure}[H]
\centering
\includegraphics[width=0.98\textwidth]{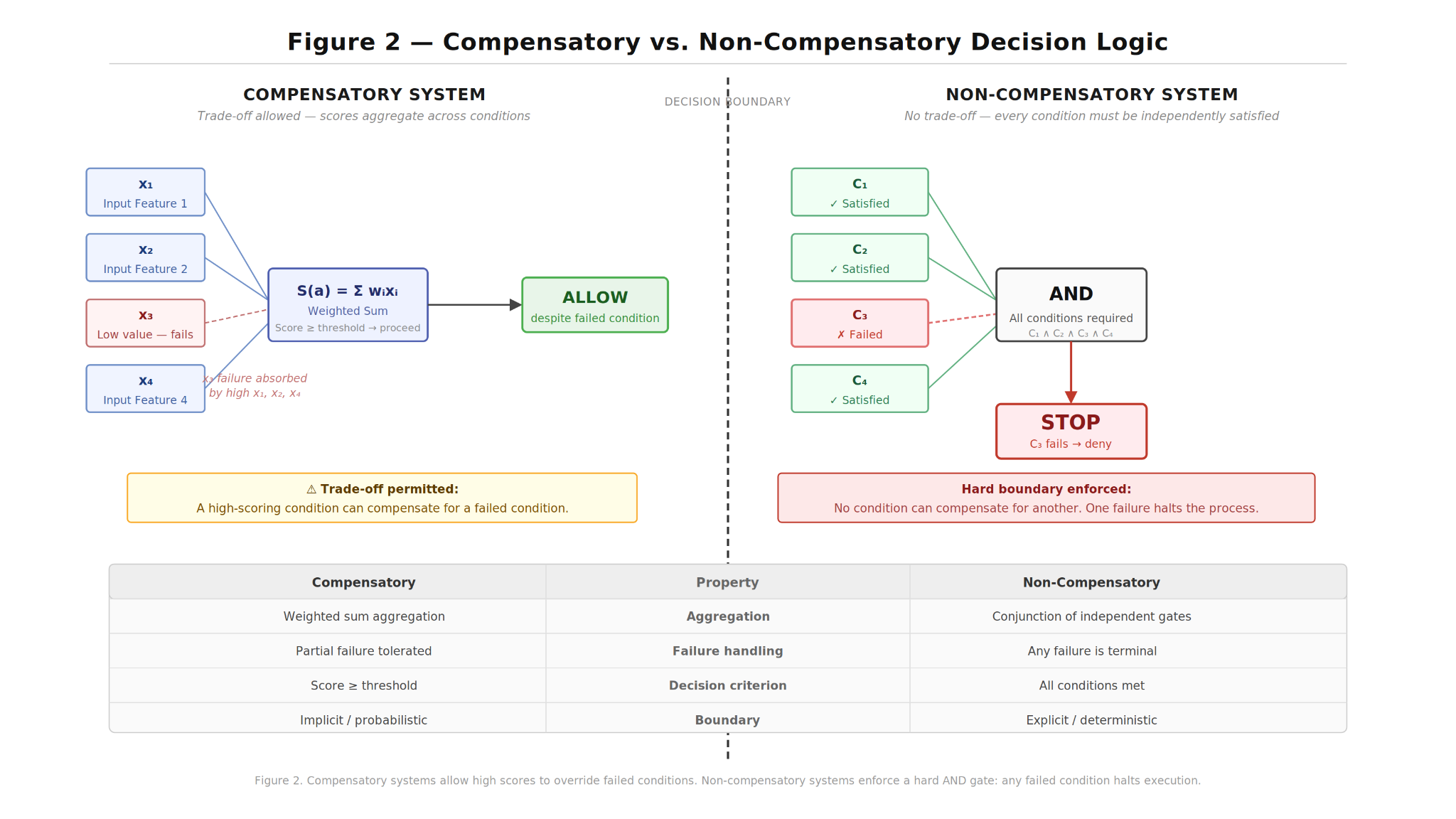}
\caption{A scoring threshold can permit compensation. A Right-to-Act boundary permits execution only inside a feasible region defined by required constraints.}
\label{fig:boundary}
\end{figure}

\begin{table}[H]
\centering
\small
\renewcommand{\arraystretch}{1.22}
\setlength{\tabcolsep}{5pt}
\caption{Comparison of decision paradigms.}
\begin{tabular}{p{3.0cm}p{4.2cm}p{6.0cm}}
\toprule
\textbf{Paradigm} & \textbf{Primary question} & \textbf{Gap relative to pre-action legitimacy} \\
\midrule
Authorization & Who is allowed? & Authorized decisions may still be structurally illegitimate. \\
Safety / alignment & Is it harmful? & Harm reduction does not imply that the decision has the right to execute. \\
Runtime governance & Does the path comply? & Path compliance may govern behavior without defining existential legitimacy. \\
Statistical certification & Is system behavior bounded? & Certification bounds system behavior, not each decision's right to exist. \\
Right-to-Act & Should the decision proceed? & Non-compensatory boundary over required structural conditions. \\
\bottomrule
\end{tabular}
\end{table}

\section{Methodological Perspective}

A publishable Right-to-Act analysis should not stop at a philosophical statement. It requires a method. The following method is deliberately abstract enough to avoid implementation disclosure while remaining testable.

\subsection{Step 1: identify executable decision classes}

The first step is to list decision classes that can produce external effects. Examples include account disablement, payment release, transaction escalation, infrastructure modification, access revocation, medical routing, or automated reporting.

\subsection{Step 2: define the action tuple}

Each decision class is represented as a structured decision object $a$. The representation must contain enough information to evaluate context, target, scope, timing, and consequence.

\subsection{Step 3: define required constraints}

For each decision class, a set of required constraints is defined. The constraints are not weights. They are necessary conditions. The purpose is not to maximize approval but to define the conditions under which execution is legitimate.

\subsection{Step 4: define non-action outcomes}

A failed Right-to-Act evaluation should not be treated as a system error. It should map to a controlled non-action state, such as:

\begin{equation}
\NONACTION \in \{\DEFER,\ \ESCALATE,\ \INFO\}.
\end{equation}

This design makes inaction explicit, auditable, and operationally meaningful.

\subsection{Step 5: evaluate against case studies}

Each case study compares at least two systems:

\begin{enumerate}[leftmargin=1.2cm]
  \item a compensatory baseline, such as a risk score or threshold classifier;
  \item a Right-to-Act boundary, based on required constraints.
\end{enumerate}

A compelling case study is one in which the baseline allows execution while the Right-to-Act boundary blocks, defers, or escalates for a structurally valid reason.

\section{Case Study: AI-Driven Account Suspension}

\subsection{Scenario}

Consider an online platform that uses AI to detect account abuse. The system proposes the following decision:

\begin{quote}
Disable user account $u$ immediately.
\end{quote}

The decision is based on multiple signals:

\begin{itemize}[leftmargin=1.2cm]
  \item repeated automated flags;
  \item similarity to known malicious patterns;
  \item a high model confidence score;
  \item prior suspicious activity associated with the account.
\end{itemize}

This is a realistic decision class because account suspension affects access, reputation, business continuity, and sometimes legal rights.

\begin{figure}[H]
\centering
\includegraphics[width=0.98\textwidth]{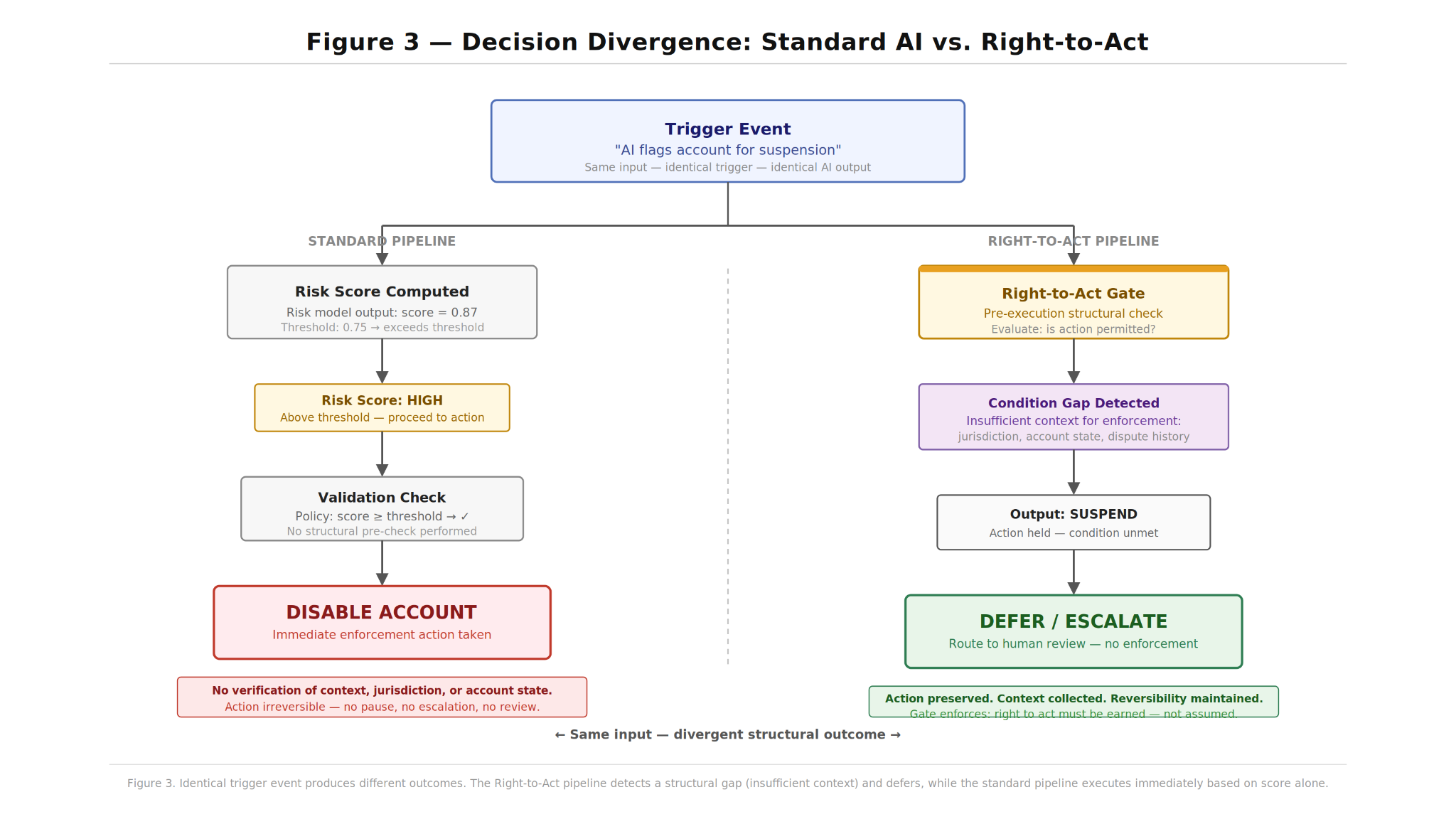}
\caption{The account suspension case highlights the difference between aggregate scoring and a non-compensatory pre-action boundary.}
\label{fig:case}
\end{figure}

\subsection{Compensatory baseline}

A conventional system may compute:

\begin{equation}
S(a)=w_1x_1(a)+w_2x_2(a)+w_3x_3(a)+w_4x_4(a),
\end{equation}

where the features represent flags, similarity, confidence, and prior history. If $S(a)\geq\theta$, the account is disabled.

Assume:

\begin{equation}
S(a)\geq\theta.
\end{equation}

The baseline outcome is therefore:

\begin{equation}
\text{Baseline}(a)=\ALLOW.
\end{equation}

\subsection{Missing structural condition}

Now suppose a required condition fails. For example, the system lacks sufficient contextual verification. The signal may be caused by a compromised device, a false correlation, coordinated reporting, or AI-generated content that resembles malicious behavior but is not causally tied to the user.

Let this condition be:

\begin{equation}
C_{context}(a)=0.
\end{equation}

Under a Right-to-Act boundary:

\begin{equation}
C_{context}(a)=0 \Rightarrow D(a)=\NONACTION.
\end{equation}

The output is therefore not immediate suspension. It may be:

\begin{equation}
D(a)=\ESCALATE \quad \text{or} \quad D(a)=\INFO.
\end{equation}

\subsection{Result}

\begin{table}[H]
\centering
\caption{Case study outcome comparison.}
\begin{tabular}{p{4.2cm}p{4.8cm}p{4.8cm}}
\toprule
\textbf{System} & \textbf{Reasoning} & \textbf{Outcome} \\
\midrule
Compensatory baseline & High aggregate score exceeds threshold & Disable account \\
Right-to-Act boundary & Required contextual condition fails & Non-action: escalate or request information \\
\bottomrule
\end{tabular}
\end{table}

The case shows that a high-confidence, policy-compliant, statistically justified decision may still fail a required legitimacy condition. The difference is not marginal. It changes the execution outcome.

\subsection{Why this case matters}

The account suspension case is useful because it reveals the exact weakness of compensation. A model can be confident, a risk score can be high, and a policy can match, yet the decision may still be invalid if a required condition is missing. This is precisely where a Right-to-Act boundary creates value: it treats missing legitimacy not as lower score, but as a blocker.

\section{Recursive Data and Certification Limits}

A further challenge arises when AI systems are evaluated on data that is itself partially generated, modified, summarized, or influenced by AI systems. In such cases, statistical validation may become self-referential. The evaluation distribution may no longer represent an external ground truth but a distribution recursively shaped by similar systems.

This does not invalidate statistical certification. It identifies a limit. If the reference distribution is shaped by AI-generated artifacts, then behavioral certification may certify conformity to a distorted distribution. A pre-action constraint boundary can mitigate this by enforcing structural requirements on individual decisions rather than relying solely on aggregate behavior.

This distinction is especially important for domains such as fraud detection, content moderation, peer review, automated compliance, and decision support systems, where AI-generated text and behavior may feed back into future evaluation data.

\section{Discussion}

The proposed framework changes the object of evaluation. Instead of asking only whether a model is accurate, safe, compliant, or certified, it asks whether a particular decision should be allowed to become executable.

\subsection{Inaction becomes a first-class outcome}

Many systems treat non-execution as failure. In a Right-to-Act protocol, non-action is a valid outcome. It may mean that the decision must be deferred, escalated, or returned for additional information.

\subsection{Auditability improves}

A non-compensatory decision boundary supports clearer audit trails because execution can be tied to the satisfaction of required conditions. When execution is blocked, the reason is not merely a low score. It is a failed requirement.

\subsection{Model independence}

The boundary can be applied to decisions generated by different models, agents, or workflows, provided that the decision can be represented structurally. This makes the framework compatible with heterogeneous AI systems.

\subsection{Limits of the present paper}

This paper does not define the full internal method for selecting, calibrating, or implementing constraints. It also does not claim that every decision can be reduced to a simple Boolean predicate without domain work. The contribution is the formal boundary and the non-equivalence result, not a universal implementation recipe.

\section{Conclusion}

AI systems increasingly act in the world. Existing systems ask important questions: who is allowed, what is risky, what is safe, what is compliant, and what can be certified. This paper identifies a missing question:

\begin{quote}
\centering
\textbf{Should this decision be allowed to exist before execution?}
\end{quote}

We defined the Pre-Action Legitimacy Gap and formalized a Right-to-Act boundary as a deterministic, non-compensatory constraint system. We showed that compensatory scoring systems cannot guarantee this property and demonstrated the distinction through an account suspension case study.

The central claim is simple:

\begin{quote}
\centering
\textbf{Legitimacy is not a score. It is a boundary.}
\end{quote}

\appendix

\section{Terminology and Legend}

\begin{table}[H]
\centering
\caption{Glossary of terms used in the paper.}
\begin{tabular}{p{3.5cm}p{10cm}}
\toprule
\textbf{Term} & \textbf{Meaning} \\
\midrule
Action / decision $a$ & A structured AI-generated or AI-routed candidate for execution. \\
Action space $A$ & The set of all candidate executable decisions. \\
Constraint $C_i$ & A required condition that must hold for execution to be legitimate. \\
Feasible region $F$ & The set of decisions satisfying all required constraints. \\
Right-to-Act & The decision boundary that determines whether a decision may proceed toward execution. \\
Non-action & A valid controlled outcome such as defer, escalate, or request information. \\
Compensatory system & A system in which positive signals may offset weak or failed signals. \\
Non-compensatory system & A system in which no positive signal may offset a failed required condition. \\
Pre-action legitimacy & The property that a decision has the right to become executable before execution occurs. \\
\bottomrule
\end{tabular}
\end{table}

\section{Minimal Decision Rule}

The minimal decision rule can be written as:

\begin{equation}
\ALLOW \iff \neg \exists i \text{ such that } C_i(a)=0.
\end{equation}

Equivalently:

\begin{equation}
\ALLOW \iff a \in F.
\end{equation}

This appendix states the core principle without exposing internal constraint definitions or implementation details.

\end{document}